\documentclass{article}

\usepackage[utf8]{inputenc}
\usepackage[T1]{fontenc}
\usepackage[margin=1.1in,letterpaper]{geometry}
\usepackage[sort&compress,numbers]{natbib}
\usepackage[colorlinks=true,citecolor=blue,breaklinks]{hyperref}
\usepackage[hyphenbreaks]{breakurl} 

\usepackage{url}
\usepackage{microtype}

\usepackage{times}

\usepackage{bbm}
\usepackage{dsfont} 
\usepackage{nicefrac}
\usepackage[super]{nth}
\usepackage{amsmath}
\usepackage{amssymb}
\usepackage{amsfonts}
\usepackage{amsthm}

\newtheorem{proposition}{Proposition}

\newcommand*\diff{\mathop{}\!\mathrm{d}}
\DeclareMathOperator*{\argmax}{arg\,max}
\DeclareMathOperator*{\argmin}{arg\,min}

\usepackage{algorithm}
\usepackage{algcompatible}
\algnewcommand\INPUT{\item[\textbf{Input:}]}
\algnewcommand\OUTPUT{\item[\textbf{Output:}]}

\newcommand{\ie}{\textit{i}.\textit{e}.\ }
\newcommand{\eg}{\textit{e}.\textit{g}.\ }

\usepackage{array}
\usepackage{multirow}
\usepackage{booktabs}

\usepackage{float}
\usepackage{epsfig}
\usepackage{graphicx}
\usepackage{caption}
\usepackage{subcaption}
\usepackage{wrapfig}

\usepackage{textcomp}

\usepackage{authblk}

\title{Open-Set Hypothesis Transfer with Semantic Consistency}

\author[1]{Zeyu Feng}
\author[1]{Chang Xu}
\author[1]{Dacheng Tao}
\affil[1]{UBTECH Sydney AI Centre, School of Computer Science, Faculty of Engineering,

The University of Sydney, Darlington, NSW 2008, Australia}
\affil[ ]{
    zfen2406@uni.sydney.edu.au, \{c.xu, dacheng.tao\}@sydney.edu.au
}
\date{}

\begin{document}

\maketitle

\begin{abstract}
  Unsupervised open-set domain adaptation (UODA) is a realistic problem where unlabeled target data contain unknown classes. Prior methods rely on the coexistence of both source and target domain data to perform domain alignment, which greatly limits their applications when source domain data are restricted due to privacy concerns. This paper addresses the challenging hypothesis transfer setting for UODA, where data from source domain are no longer available during adaptation on target domain. We introduce a method that focuses on the semantic consistency under transformation of target data, which is rarely appreciated by previous domain adaptation methods. Specifically, our model first discovers confident predictions and performs classification with pseudo-labels. Then we enforce the model to output consistent and definite predictions on semantically similar inputs. As a result, unlabeled data can be classified into discriminative classes coincided with either source classes or unknown classes. Experimental results show that our model outperforms state-of-the-art methods on UODA benchmarks.
\end{abstract}

\section{Introduction}

Unsupervised domain adaptation (UDA) methods have achieved remarkable advancement in reducing the effort of laborious data annotation on a target dataset of interest by utilizing labeled dataset that has a similar distribution. In recent years, unsupervised open-set domain adaptation (UODA), which assumes that the target domain has a richer label space, has received considerable attentions as it is a more realistic scenario. Many techniques have been proposed for this problem, for example, learning a margin between known and unknown instances in target domain~\cite{Saito_2018_ECCV}, or detecting confident known target instances based on knowledge learned from source domain and then aligning the distribution~\cite{Liu_2019_CVPR}.

Despite the great success have been achieved, a shortcoming of these approaches is the requirement of training with source domain data during adaptation, which does not apply to the cases where source data are not available from provider, such as data privacy concerns and copyright laws. The algorithm is unable to access the source domain data if these data contain private sensitive information, data are distributed on different devices, or the provider are prohibited to share the data. Thus, it is better for algorithms to adapt well on target domain with only a trained source model to protect data privacy. This also has the benefit of efficiently delivering only model parameters as opposed to large datasets.

Without source domain data, algorithms have to extract information from either the trained source model or target data. Previous methods~\cite{pmlr-v37-long15,pmlr-v37-ganin15,JMLR:v17:15-239} that rely on various forms of distribution alignment no longer work in this scenario. Focusing on building a stronger pre-trained model,~\citet{inheritable_cvpr_2020} propose a feature splicing technique for source model training so that it can better generalize to unseen categories in target domain, while requiring extra training pattern for source model provider. From another perspective,~\citet{liang2020really} aims to learn a source like representation on target domain data for pre-trained classifier, which is mainly designed for close-set DA. Although some progress have been achieved through these insights, structural properties in target domain data are not fully explored, from which semantic constraints can be constructed to benefit adaptation and separating known/unknown instances.

In this paper, we propose an UODA method that exploits both pre-trained model and target input structures. The learner on target domain (client) only requires a model trained under a standard softmax cross-entropy loss from the source domain provider (vendor). Firstly, the pre-trained model is used to generate pseudo-labels of confident target instances for classification, ignoring the instances with uncertain predictions caused by domain gap, creating a confident decision boundary on target domain. Secondly we train the model to make consistent predictions on similar inputs in target domain by observing that transformed copies of the same input image often share the same semantic, which can discriminate different semantics on target domain. Specifically, we maximize the mutual information between a pair of the similar input on extended label space. The objective is straightforward to implement and easy to optimize. The proposed two modules can mutually benefit each other and improve the adaptation performance.

We use the term open-set hypothesis transfer to indicate the setting wherein the adaptation on an open-set target domain does not use source domain data. To demonstrate the effectiveness of our approach, we conduct experiments on standard UODA benchmarks. Ablation study and hyper-parameters analysis verify the effectiveness of each component. Our method achieves state-of-the-art performances on Office-31, Office-Home and VisDA-2017 datasets.

\section{Related work}

\subsection{Open-set domain adaptation} The main challenge in UODA task is the existence of instances with novel classes in target domain, which will cause negative transfer if they are not properly handled. ATI-$\lambda$~\cite{Busto_2017_ICCV,8531764} first presents a method to assign target instances either source class labels or unknown label based on feature distance, and to employ open-set support vector machine~\cite{6809169} to inference on target domain.~\citet{baktashmotlagh2018learning} propose to model known and unknown samples into different subspace to separate them. Adversarial distribution alignment, popularly adopted in close-set DA, is also adapted to UODA by weighting each target instances in adversarial learning to reduce the impact of possible unknown class instances during alignment~\cite{Liu_2019_CVPR}. A different but popularly adopted solution to this problem is separating known and unknown target classes by adversarially learning a boundary between them~\cite{Saito_2018_ECCV}. This approach can be strengthened by considering aligning mean feature prototypes on class-level~\cite{Feng_2019_ICCV}.  

For source data agnostic adaptation,~\citet{inheritable_cvpr_2020} propose a feature splicing technique to train the source model, which generates out-of-distribution features for downstream unknown instance detection.~\cite{inheritable_cvpr_2020} requires the vendor to train the model with extra effort. Contrary to this approach, we only use a standard cross-entropy trained source model and focus on the target data structure.~\cite{liang2020really} assumes that target data representation is expected to lead definite predictions under a pre-trained classifier, and hence trains the representation function. However, the extension of this work on UODA task does not rely on this principle to detect unknown target instances. In this paper, we rely on semantic similarities in target data to detect every known and unknown target classes.

\subsection{Transformation consistency} Consistency among data transformations has been successfully used as an optimization objective in semi-supervised learning. In~\cite{NIPS2016_6333}, multiple random transformations with disturbances of the data are constrained to have similar prediction by minimizing the squared euclidean distance between outputs. It often achieves effective improvement when combined with other insights~\cite{laine2016temporal,NIPS2019_8749}. In recent years this transformation consistency is widely used in self-supervised representation learning under the form of contrastive loss function~\cite{arora2019theoretical}. Instead of minimizing the distance, it compares transformations with negative inputs (usually other data instances)~\cite{Wu_2018_CVPR,Images_2019_infomax,bachman2019amdim,tian2019contrastive,he2019momentum}, satisfying the relationship that transformation copies are more semantically similar than other data instances.

The structural information in data transformation has been explored in the filed of domain adaptation. Particularly,~\cite{sun2019unsupervised} and~\cite{8883232} consider self-supervisions for traditional close-set domain adaptation tasks, nevertheless these self-supervisions are used as a form of auxiliary domain alignment, which relies on the use of source domain data. On the contrary, our method explores the supervision from transformation consistency on target open-set domain data and directly connects it to classification with known and unknown categorical outputs.

\section{Open-set hypothesis transfer}

\begin{figure}[t]
	\centering
	\includegraphics[width=0.99\textwidth]{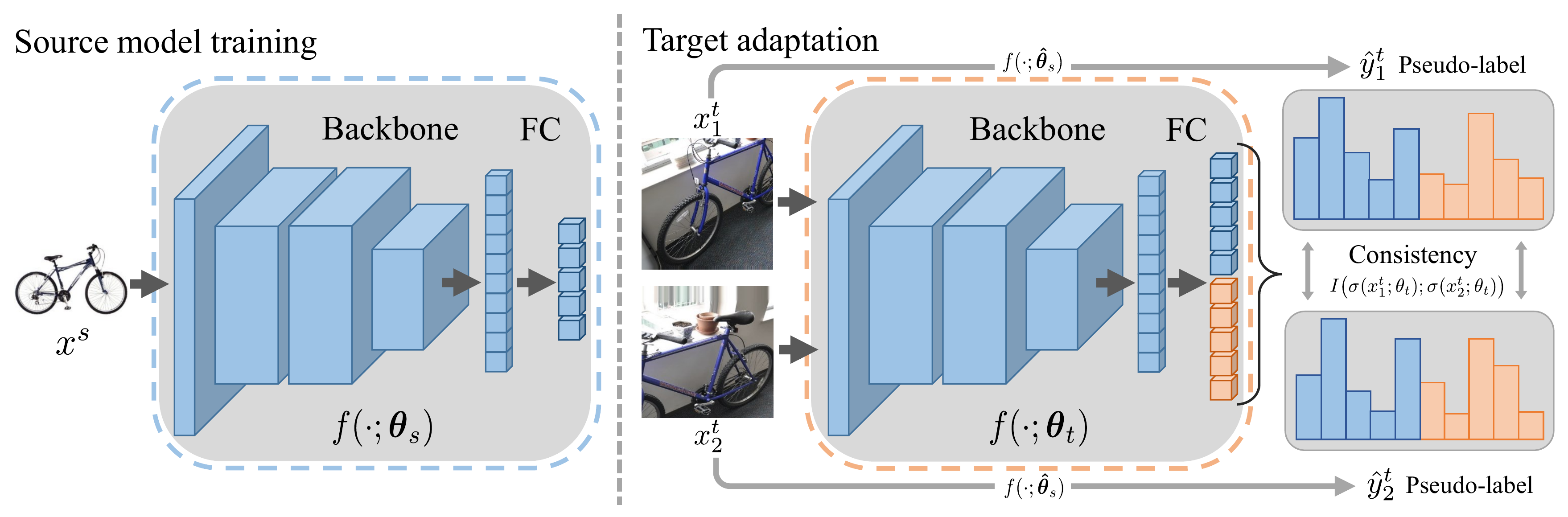}
	\caption{Illustration of the proposed framework. In our method for open-set hypothesis transfer, a network is first trained in source domain by standard cross-entropy loss. Then the pre-trained model is adapted to a target domain. We use the pre-trained model to generate pseudo-labels on confident target instances. Then the predictions of a pair of target instances created using transformation are constrained to be similar by maximizing their mutual information.}
	\label{fig:3.1}
\end{figure}

In this section, we describe the details of our approach. We first set up notations and introduce the problem setting. Then we explain the proposed objective functions of pseudo-label generation and transformation consistency, respectively. Figure~\ref{fig:3.1} illustrates the overall framework.

\subsection{Source domain model}

\noindent \textbf{Notation} \quad
Let $\mathcal{D}_s$ and $\mathcal{D}_t$ denote source and target distributions on data space $\mathcal{X}\times\mathcal{Y}$, and $\mathcal{D}_{s,X}$ and $\mathcal{D}_{t,X}$ be their marginal distribution on input space $\mathcal{X}$, respectively. Given a labeled source domain dataset $D_s=\{(x_i^s,y_i^s)\}_{i=1}^{n_s}\sim\mathcal{D}_s^{n_s}$ and a target unlabeled domain dataset $D_t=\{x_j^t\}_{j=1}^{n_t}\sim\mathcal{D}_{t,X}^{n_t}$, we learn a predictive model $f(\cdot;\theta)$ that outputs a score for each class $y\in\mathcal{Y}$ of an input $x\in\mathcal{X}$. In unsupervised open-set domain adaptation (UODA), the label space of source domain $C_s\subseteq\mathcal{Y}$ and target domain $C_t\subseteq\mathcal{Y}$ satisfy $C_s\subseteq C_t$.

The adaptation algorithm can only obtain a pre-trained model $\hat{f}_s$ from source domain (the vendor). For the $\lvert C_s\lvert$-way classification on source domain, the model $f(\cdot)\in\mathbb{R}^{\lvert C_s\lvert}$ is pre-trained with cross-entropy loss:
\begin{equation}
    L_S(\theta_s) = \frac{1}{n_s}\sum_{i=1}^{n_s}\ell_{ce}\big(\sigma(f(x_i^s;\theta_s)),y_i^s\big),
\end{equation}
where $\sigma$ is the softmax function and $\ell_{ce}$ is the cross-entropy loss for multiclass classification. Then the pre-trained model $\hat{f}_s=f(\cdot;\hat{\theta}_s)$, where $\hat{\theta}_s\in\argmin_{\theta_s} L_S(\theta_s)$, can be delivered to the target domain learner.

\subsection{Pseudo-label generation}

\label{sec:pseudo}

Due to data distribution shift, a model pre-trained on source domain will suffer degradation of performance on another domain. Particularly, the pre-trained model $\hat{f}_s$ cannot be used for an unseen open-set domain in UODA problem since both marginal input distribution $p(x)$ and conditional output distribution $p(y\vert x)$ change. The input distribution shift (\eg image style change) alone can reduce the accuracy of the predictions on known-class instances in target domain. In the meantime, target data of novel classes not belonging to source domain have to be picked out and classified as `unknown', which is not possible directly by the available source model $\hat{f}_s$.

To overcome these limitations and make $\hat{f}_s$ usable, we observe that although UODA task is challenging, it is possible to split $p(y\vert x)$ into $p(y_{known}\vert x_{known})$ and $p(y_{unknown}\vert x_{unknown})$, and tackle them individually. If some instances $x_{unknown}$ and $x_{known}$ are selected with high probability being true un/known, we are able to label $x_{unknown}$ as unknown since $p(y_{unknown}\vert x_{unknown})=1$ ($p(y_{known}\vert x_{unknown})=0$, respectively), and directly use $\hat{f}_s$ on $x_{known}$ through adaptation in the close-set setting since $p(y_{known}\vert x_{known})$ stays unchanged in UODA. To verify the feasibility of this separation and selection, we compare the predictions of true un/known target instances. Figure~\ref{fig:entropy} in section~\ref{sec:reliability} will show that known and unknown are statistically different under the confidence of $\hat{f}_s$'s prediction. Based on these observations, we leverage the source model $\hat{f}_s$ by taking the confidence of prediction into consideration, because $\hat{f}_s$ can still classify some target data successfully if the confidence of prediction is high. For example, if the output probability of certain class is far larger than any other classes, then this classification is possibly correct. On the other hand, if no probability of any class dominates the prediction, (\ie prediction is quite balanced), then this instance very likely belongs to unknown classes in target domain. 

We use $\hat{f}_s(\cdot)$ pre-trained on source domain to generate pseudo labels for confident target data. It has been shown that pseudo labels can be effectively harnessed in conventional close-set domain adaptation scenario~\cite{selfpaced,zhang2018collaborative,choi2019pseudo}. Formally, we use information entropy of predicted probability as measure of confidence. The entropy of the prediction on source domain categories $H\big(\sigma(\hat{f}_s(\cdot))\big)$ will be small if the prediction is confident, while it will be large when uncertain~\cite{You_2019_CVPR}. A target instance $x^t$ is considered to be a confident known source class instance if $H\big(\sigma(\hat{f}_s(x^t))\big)$ is smaller than a threshold $\delta_k$, and a confident unknown class instance if $H\big(\sigma(\hat{f}_s(x^t))\big)$ is larger than a threshold $\delta_u$. We will also explore other variants of the confidence measure in section~\ref{sec:alter}.

Thus, the pseudo-label $\hat{y}^t$ of a target instance $x^t$ is
\begin{equation}
    \hat{y}^t = 
    \begin{cases} 
        \argmax_{c\in C_s}\hat{f}_c(x^t)    & H\big(\sigma(\hat{f}_s(x^t))\big) \leq \delta_k \\
        \lvert C_S\rvert+1                  & H\big(\sigma(\hat{f}_s(x^t))\big) \geq \delta_u \\
        None                                & other 
    \end{cases}
\end{equation}
Target instances with uncertain predictions are discarded and no pseudo-labels are generated for them. Practically, the threshold are dependent on the number of source classes $\lvert C_s\rvert$ because of the dependence of entropy $H$ on the number of elements. Note that the maximum value of an entropy is $\log N$, we set $\delta_u=\nicefrac{\log\lvert C_s\rvert}{2}$ and $\delta_k=0.1\delta_u$ in this paper. Denote the pseudo-labeled target dataset with known classes and unknown class by $D_{t,k}=\big\{(x^t,\hat{y}^t)\vert H\big(\sigma(\hat{f}_s(x^t))\big) \leq \delta_k\big\}$ and $D_{t,u}=\big\{(x^t,\hat{y}^t)\vert H\big(\sigma(\hat{f}_s(x^t))\big) \geq \delta_u\big\}$, respectively. Next, we can train the model on $D_{t,k}$ and $D_{t,u}$ by minimizing the cross-entropy loss.

In order to make $\hat{f}_s$ usable for explicit inference on the target domain, we expand the model so that it has $\lvert C_s\lvert+K$ outputs and $f(\cdot;\theta_t)\in\mathbb{R}^{\lvert C_s\lvert+K}$. The expanded parameters $\theta_t$ consists of the original part $\theta_s$ and the expanded part $\theta_e$: $\theta_t=\theta_s\cup\theta_e$. For example, this could be expanding the last fully connected layer of a neural network to output $\lvert C_s\lvert+K$ logits. The first $\lvert C_s\lvert$ logits are those used for classification on source dataset. Here $K$ is an approximation of the number of target unknown classes $\lvert C_t\rvert-\lvert C_s\rvert$. Note that the expanded model has more than one output ($K$) for unknown classes. For this part, in order to represent the probability of the unknown class, we simply add the probability values on these $K$ indices. The objective with cross-entropy loss is then
\begin{equation}
    L_P(\theta_t) = \frac{1}{\lvert D_{t,k}\rvert}\sum_{x^t\in D_{t,k}}\ell_{ce}(\sigma(f(x^t;\theta_t)),\hat{y}^t) - \frac{1}{\lvert D_{t,u}\rvert}\sum_{x^t\in D_{t,u}}\log\sum_{c=\lvert C_s\lvert+1}^{\lvert C_s\lvert+K}\sigma_c(f(x^t;\theta_t)),
    \label{eq:pseudo_label}
\end{equation}
where $v_c(\cdot)$ means the $c$-th component of the vector outputted by the function $v(\cdot)$. By optimizing this objective, the model will be trained to further classify confident target instances and enlarge the decision margin on these confident instances.

\subsection{Transformation consistency}

In order to extract the semantic information of unlabeled target data as much as possible, we create different transformed copies of one input and enforce that their predictions are consistent. This transformation can be achieved through data augmentation techniques like noise, crop and affine transformation, which does not alter class semantics of inputs. The predictions of these copies should be both similar and definite, so that the model is capable of distinguishing both similar and dissimilar class semantics. Discrepancy measure like KL divergence and $\ell_2$ distance has been in close-set setting. However, in the open-set scenario, target domain contains more than one unknown classes. Therefore, besides giving scores for $\lvert C_s\rvert$ source classes, the model should also give $K$ more scores, which could hurt the model by causing ambiguity both between known and unknown classes and between different unknown classes if simply enforcing these aforementioned distances.

Let the random variable $\tilde{y}^t$ denote the predicted class of $x^t$ by $f$ according to $\argmax_{c\in[\lvert C_s\lvert+K]}f_c(x^t)$. Then the probability of $\tilde{y}^t$ being class $c$ conditioned on $x^t$ is $p(c\vert x^t)=\sigma_c(f(x^t;\theta_t))$, where we drop the term $\tilde{y}^t$ when it is clear from the context. We wish to make the prediction of $x^t$ and its augmentation $x_+^t$ as similar as possible, which can be achieved through maximizing the mutual information $I(\tilde{y}^t;\tilde{y}_+^t)$ between $\tilde{y}^t$ and $\tilde{y}_+^t$:
\begin{equation}
    \label{eq:mi}
	\min_{\theta_t}-I(\tilde{y}^t;\tilde{y}_+^t).
\end{equation}
With this objective, the model has to give similar and definite predictions for each pair~\cite{Ji_2019_ICCV}. Since the transformation guarantees the semantic similarity in every pair, then semantically similar images (usually in same class) have to be classified into same category. As a result, images similar to source classes are pushed to the source categories. For the unknown target classes, the model can discover instances that are dissimilar to source classes and classify these to different unknown categories. The following theorem gives an explanation on why maximizing this mutual information can lead to discriminative predictions.
\begin{proposition}
	\label{thm:mi}
    Assume that transformation $x_+^t$ does not lose the information of its true label $y^t$ and contains no other information in $x^t$, \ie $x_+^t=\argmin_{x'}I(x^t;x'),\,\textrm{s.t.}\, I(x^t;y^t)=I(x';y^t)$, then
    \begin{equation}
        I(\tilde{y}^t;\tilde{y}_+^t) \leq I(\tilde{y}^t;y^t).
    \end{equation}
\end{proposition}
\noindent In other words, if the augmentation is perfect in terms of the definition in Proposition~\ref{thm:mi} with the only information shared between $x^t$ and $x_+^t$ is $y^t$, then with objective~\ref{eq:mi} we are maximizing a lower bound of the mutual information between the prediction $\tilde{y}^t$ and its true label $y^t$, making the prediction discriminative towards target label.

Recall that $I(\tilde{y}^t;\tilde{y}_+^t)=H(\tilde{y}^t)-H(\tilde{y}^t\vert \tilde{y}_+^t)$. When $H(\tilde{y}^t)$ is maximized, then the prediction can have a more balanced distribution on both source classes and unknown classes, potentially avoiding the degenerate solution where all target data are classified into only a subset of all classes, especially the known ones. Therefore, in practice we seek an objective $\beta H(\tilde{y}^t)-H(\tilde{y}^t\vert \tilde{y}_+^t)$ that balances between the entropy and the conditional entropy term. Large $\beta$ will encourage the model to predict target data into different classes with equal proportion. According to the symmetry of this factorization (section~\ref{sec:app_mi}), the objective can be written as
\begin{align}
& L_{C,\beta}(\theta_t) = -I_{\beta}(\tilde{y}^t;\tilde{y}_+^t) = -\sum_{c=1}^{\lvert C_s\lvert+K}\sum_{c_+=1}^{\lvert C_s\lvert+K} p(\tilde{y}^t=c,\tilde{y}_+^t=c_+) \log\frac{p(\tilde{y}^t=c,\tilde{y}_+^t=c_+)}{\big(p(\tilde{y}^t=c)p(\tilde{y}_+^t=c_+)\big)^{\frac{\beta+1}{2}}} \label{eq:mi_raw} \\ 
= & -\sum_{c=1}^{\lvert C_s\lvert+K}\sum_{c_+=1}^{\lvert C_s\lvert+K} \int p(c\vert x^t)p(c_+\vert x_+^t)\diff p(x^t,x_+^t) \log\frac{\int p(c\vert x^t)p(c_+\vert x_+^t)\diff p(x^t,x_+^t)}{\big(\int p(c\vert x^t)\diff p(x^t)\int p(c_+\vert x_+^t)\diff p(x_+^t)\big)^{\frac{\beta+1}{2}}}. \notag
\end{align}

In computation, the joint and marginal distributions in Eq.~(\ref{eq:mi_raw}) are approximated using empirical distributions from datasets to calculate the estimate $\hat{I}_{\beta}(\tilde{y}^t;\tilde{y}_+^t)$, where $p(c,c_+)\approx P_{c,c_+}:=\frac{1}{n_t}\sum_{j=1}^{n_t}p(c\vert x_j^t)p(c_+\vert x_{j,c_+}^t)$ and $p(c)\approx P_c:=\frac{1}{n_t}\sum_{j=1}^{n_t}p(c\vert x_j^t)$. The objective for optimization is
\begin{equation}
\hat{L}_{C,\beta}(\theta_t) = -\hat{I}_{\beta}(\tilde{y}^t;\tilde{y}_+^t) = -\sum_{c=1}^{\lvert C_s\lvert+K}\sum_{c_+=1}^{\lvert C_s\lvert+K}P_{c,c_+}\log\frac{P_{c,c_+}}{(P_cP_{c_+})^{\frac{\beta+1}{2}}}.
\label{eq:consistency}
\end{equation}
Although these distributions are approximated separately, the following proposition shows its consistency.
\begin{proposition}
	\label{thm:consistency}
	The estimator $\hat{I}_{\beta}(\tilde{y}^t;\tilde{y}_+^t)$ of $I_{\beta}(\tilde{y}^t;\tilde{y}_+^t)$ is consistent.
\end{proposition}

\noindent \textbf{Complete model}\quad The complete objective functions consist of both pseudo-label classification~(Eq.~(\ref{eq:pseudo_label})) and transformation consistency~(Eq.~(\ref{eq:consistency}))
\begin{equation}
    \min_{\theta_t} \alpha_pL_P(\theta_t) + \alpha_c\hat{L}_{C,\beta}(\theta_t),
\end{equation}
where $\alpha_p$ and $\alpha_c$ are hyperparameters for loss balancing. Before training on target domain, $f(\cdot;\theta_t)$ is initialized by $\hat{f}_s$.

During inference, the probability $\sigma_{1:\lvert C_s\lvert}(f(x^t))$ and $\sum_{c=\lvert C_s\lvert+1}^{\lvert C_s\lvert+K}\sigma_c(f(x^t))$ are compared to classify a target instance either into a known class or the unknown class.

\section{Experiments}

We conduct experiments on standard UODA datasets to demonstrate the effectiveness of our approach and compare it with state-of-the-art methods. The standard datasets include Office-31~\cite{10.1007/978-3-642-15561-1_16}, Office-Home~\cite{Venkateswara_2017_CVPR} and VisDA-2017~\cite{peng2017visda}. Office-31 has 3 domains (Amazon, DSLR and Webcam) with 4,652 images in 31 shared classes. Office-Home contains 4 domains (Artistic, Clipart, Product and Real-world) with around 15k images in 65 shared classes. Tasks are constructed from each two domains in both directions, forming 6 and 12 adaptation scenarios, respectively. For VisDA-2017 the task is adaptation from Synthetic domain (150k synthetic images) to Real domain (55k real images). We follow previous work~\cite{Saito_2018_ECCV,Liu_2019_CVPR,inheritable_cvpr_2020} to partition known and unknown classes in all domains.

\subsection{Implementation details}

Following the experimental settings in~\cite{Liu_2019_CVPR} and~\cite{inheritable_cvpr_2020}, we choose ResNet-50 as the backbone network for Office-31 and Office-Home, and VGG-16 for VisDA-2017 dataset. The value of $K$ is set as 70 and $\beta$ is set as 1.3 throughout all datasets. The influence of these two hyperparameters is analyzed in Section~\ref{sec:hyper}. We use a batch-size of 64 with 32 on source domain and 32 on target domain. All networks are trained with momentum SGD using a learning rate 0.0005, a momentum 0.9 and a weight decay 0.0005. On source domain, the network is initialized with an ImageNet pre-trained model as it is common practice. ImageNet pre-trained layers have learning rate divided by 10. On target domain, the network is initialized with the source domain pre-trained model for the inherited parameters ($\theta_s$), while the expanded parameters ($\theta_e$) are trained from scratch. These learning parameters and setting are same as~\cite{Liu_2019_CVPR}, which is quite standard in open-set domain adaptation. For all experiments, we set the loss balancing parameter $\alpha_p=0.1$ and $\alpha_c=1$.

To evaluate performance, average per-class accuracy including all unknown as one class (OS) and average per-class accuracy only on the known classes (OS$^*$) are reported. Besides OS and OS$^*$, we also report the accuracy on the whole target domain for some comparisons. Since OS and OS$^*$ are averages of per-class accuracies, they do not reflect the overall precision when classes are imbalanced.

\subsection{State-of-the-art comparison}

\begin{table}
	\small
	\centering
	\caption{Average classification accuracy ($\%$) on Office-31 (ResNet-50).}
	\label{table:office31}
	\renewcommand{\tabcolsep}{3.5pt}
	\begin{tabular}{lccccccccccccccc} \toprule
		\multirow{2}{*}{\textbf{Method\textbackslash Task}} && \multicolumn{2}{c}{A$\rightarrow$D} & \multicolumn{2}{c}{A$\rightarrow$W} & \multicolumn{2}{c}{D$\rightarrow$A} & \multicolumn{2}{c}{D$\rightarrow$W} & \multicolumn{2}{c}{W$\rightarrow$A} & \multicolumn{2}{c}{W$\rightarrow$D} & \multicolumn{2}{c}{Avg} \\ \cmidrule{3-16}
		&& OS & OS$^*$ & OS & OS$^*$ & OS & OS$^*$ & OS & OS$^*$ & OS & OS$^*$ & OS & OS$^*$ & OS & OS$^*$ \\ \midrule
		ResNet									&& 85.2 & 85.5 & 82.5 & 82.7 & 71.6 & 71.5 & 94.1 & 94.3 & 75.5 & 75.2 & 96.6 & 97.0 & 84.2 & 84.4 \\
		RTN~\cite{NIPS2016_6110}				&& 89.5 & 90.1 & 85.6 & 88.1 & 72.3 & 72.8 & 94.8 & 96.2 & 73.5 & 73.9 & 97.1 & 98.7 & 85.4 & 86.8 \\
		DANN~\cite{pmlr-v37-ganin15}			&& 86.5 & 87.7 & 85.3 & 87.7 & 75.7 & 76.2 & \textbf{97.5} & \textbf{98.3} & 74.9 & 75.6 & \textbf{99.5} & \textbf{100.0} & 86.6 & 87.6 \\
		OpenMax~\cite{Bendale_2016_CVPR}		&& 87.1 & 88.4 & 87.4 & 87.5 & 83.4 & 82.1 & 96.1 & 96.2 & 82.8 & 82.8 & 98.4 & 98.5 & 89.0 & 89.3 \\
		ATI-$\lambda$~\cite{Busto_2017_ICCV}	&& 84.3 & 86.6 & 87.4 & 88.9 & 78.0 & 79.6 & 93.6 & 95.3 & 80.4 & 81.4 & 96.5 & 98.7 & 86.7 & 88.4 \\
		OSBP~\cite{Saito_2018_ECCV}				&& 88.6 & 89.2 & 86.5 & 87.6 & 88.9 & 90.6 & \underline{97.0} & 96.5 & 85.8 & 84.9 & 97.9 & 98.7 & 90.8 & 91.3 \\
		STA~\cite{Liu_2019_CVPR}				&& 93.7 & 96.1 & 89.5 & 92.1 & 89.1 & \underline{93.5} & \textbf{97.5} & 96.5 & 87.9 & 87.4 & \textbf{99.5} & \underline{99.6} & 92.9 & 94.1 \\ \midrule
		FS~\cite{inheritable_cvpr_2020}			&& \underline{94.2} & \underline{97.1} & \underline{91.3} & \underline{93.2} & \underline{90.1} & 91.5 & 96.5 & 97.4 & \underline{88.7} & \underline{88.1} & \textbf{99.5} & 99.4 & \underline{93.4} & \underline{94.5} \\ \midrule
		Ours									&& \textbf{96.6} & \textbf{97.7} & \textbf{95.7} & \textbf{96.4} & \textbf{93.1} & \textbf{93.6} & \textbf{97.5} & \underline{98.0} & \textbf{92.9} & \textbf{93.6} & \underline{99.3} & \textbf{100.0} & \textbf{95.8} & \textbf{96.6} \\ \bottomrule
	\end{tabular}
\end{table}

\begin{table}
	\small
	\centering
	\caption{Average classification accuracy OS ($\%$) on Office-Home (ResNet-50).}
	\label{table:officehome}
	\renewcommand{\tabcolsep}{3.5pt}
	\begin{tabular}{lccccccccccccccc} \toprule
		\textbf{Method\textbackslash Task} && A$\rightarrow$C & A$\rightarrow$P & A$\rightarrow$R & C$\rightarrow$A & C$\rightarrow$P & C$\rightarrow$R & P$\rightarrow$A & P$\rightarrow$C & P$\rightarrow$R & R$\rightarrow$A & R$\rightarrow$C & R$\rightarrow$P & Avg \\ \midrule
		ResNet          						&& 53.4 & 69.3 & 78.7 & 61.4 & 61.8 & 71.0 & 64.0 & 52.7 & 74.9 & 70.0 & 51.9 & 74.1 & 65.3 \\ 
		ATI-$\lambda$~\cite{Busto_2017_ICCV}	&& 55.2 & 69.1 & 79.2 & 61.7 & 63.5 & 72.9 & 64.5 & 52.6 & 75.8 & 70.7 & 53.5 & 74.1 & 66.1 \\
		DANN~\cite{pmlr-v37-ganin15}			&& 54.6 & 69.5 & 80.2 & 61.9 & 63.5 & 71.7 & 63.3 & 49.7 & 74.2 & 71.3 & 51.9 & 72.9 & 65.4 \\
		OSBP~\cite{Saito_2018_ECCV}				&& 56.7 & 67.5 & 80.6 & 62.5 & 65.5 & 74.7 & 64.8 & 51.5 & 71.5 & 69.3 & 49.2 & 74.0 & 65.7 \\
		OpenMax~\cite{Bendale_2016_CVPR}		&& 56.5 & 69.1 & 80.3 & \underline{64.1} & 64.8 & 73.0 & 64.0 & 52.9 & 76.9 & 71.2 & 53.7 & 74.5 & 66.7 \\
		STA~\cite{Liu_2019_CVPR}				&& 58.1 & 71.6 & \underline{85.0} & 63.4 & 69.3 & 75.8 & 65.2 & 53.1 & 80.8 & \underline{74.9} & 54.4 & \textbf{81.9} & 69.5 \\ \midrule
		FS~\cite{inheritable_cvpr_2020}			&& 60.1 & 70.9 & 83.2 & 64.0 & 70.0 & 75.7 & \underline{66.1} & 54.2 & 81.3 & \underline{74.9} & 56.2 & 78.6 & 69.6 \\ 
		SHOT~\cite{liang2020really}				&& \underline{60.5} & \textbf{80.4} & 82.6 & 59.2 & \underline{73.6} & \underline{77.2} & 63.4 & \underline{54.7} & \underline{82.3} & 69.5 & \underline{61.8} & \underline{81.8} & \underline{70.6} \\ \midrule
		Ours            						&& \textbf{60.6} & \underline{80.1} & \textbf{86.5} & \textbf{71.8} & \textbf{74.1} & \textbf{81.6} & \textbf{72.5} & \textbf{59.1} & \textbf{83.7} & \textbf{77.0} & \textbf{62.2} & \underline{81.8} & \textbf{74.3} \\ \bottomrule
	\end{tabular}
\end{table}

\begin{table}
    \small
    \centering
    \caption{Classification accuracy ($\%$) on VisDA-2017 (VGGNet).}
    \label{table:visda}
    \renewcommand{\arraystretch}{0.9}
    \begin{tabular}{lccccccccc} \toprule
        \multirow{2}{*}{\textbf{Method\textbackslash Class}}  && \multicolumn{8}{c}{Synthetic$\rightarrow$Real} \\ \cmidrule{3-10}
                                                && bicycle & bus & car & m-cycle & train & truck & OS & OS$^*$ \\ \midrule
        OSVM~\cite{OSVM}						&& 31.7 & 51.6 & \underline{66.5} & 70.4 & \textbf{88.5} & 20.8 & 52.5 & 54.9 \\
        MMD+OSVM        						&& 39.0 & 50.1 & 64.2 & 79.9 & 86.6 & 16.3 & 54.4 & 56.0 \\
        DANN+OSVM       						&& 31.8 & 56.6 & \textbf{71.7} & 77.4 & \underline{87.0} & 22.3 & 55.5 & 57.8 \\
        ATI-$\lambda$~\cite{Busto_2017_ICCV}	&& 46.2 & 57.5 & 56.9 & 79.1 & 81.6 & \underline{32.7} & 59.9 & 59.0 \\
        OSBP~\cite{Saito_2018_ECCV}	            && 51.1 & 67.1 & 42.8 & 84.2 & 81.8 & 28.0 & 62.9 & 59.2 \\
        STA~\cite{Liu_2019_CVPR}				&& 52.4 & \textbf{69.6} & 59.9 & \textbf{87.8} & 86.5 & 27.2 & 66.8 & 63.9 \\ \midrule
        FS~\cite{inheritable_cvpr_2020}			&& \underline{53.5} & 69.2 & 62.2 & \underline{85.7} & 85.4 & 32.5 & \underline{68.1} & \underline{64.7} \\ \midrule
        Ours            						&& \textbf{76.4} & \underline{69.4} & 52.1 & 81.5 & 76.7 & \textbf{43.6} & \textbf{70.8} & \textbf{66.6} \\ \bottomrule
    \end{tabular}
\end{table}

\noindent \textbf{Office-31}\quad Results on Office-31 are shown in Table~\ref{table:office31}. Due to space limitation, the standard deviations are reported in appendix. Here we mainly compare with the method of training with feature-splicing~\cite{inheritable_cvpr_2020} (the FS entry in the table). The proposed model outperforms Feature-splicing as well as those previous methods in the setting where source data are available. 

\noindent \textbf{Office-Home}\quad Table~\ref{table:officehome} shows the experimental results on Office-Home dataset. We compare with two methods for the hypothesis transfer UODA setting, Feature-splicing~\cite{inheritable_cvpr_2020} and SHOT~\cite{liang2020really}, and reports results on OS following previous work. Nevertheless, we report the standard deviation and OS* result of our method in appendix. From the table we can see our approach achieves improved performance both in most adaptation scenarios and averagely.

\noindent \textbf{VisDA-2017}\quad Table~\ref{table:visda} summarizes the per-class accuracies, OS and OS$^*$ scores on VisDA-2017 dataset. It can be seen from the table that our method achieves a rather balanced per-class accuracies, while many previous methods have either high or low accuracy for some classes. Moreover, VisDA is a large-scale domain adaptation dataset with large domain gap. These results demonstrate that our approach is capable of scaling to challenging scenarios.

\subsection{Analysis}

\subsubsection{Ablation study}

\begin{table}
    \small
    \centering
	\caption{Ablation study of each component on Office-31 (ResNet-50).}
	\label{table:ablation}
    \renewcommand{\tabcolsep}{1.2pt}
    \begin{tabular}{lccccccccccccccccccccc} \toprule
        \multirow{2}{*}{\textbf{M}} & \multicolumn{3}{c}{A$\rightarrow$D} & \multicolumn{3}{c}{A$\rightarrow$W} & \multicolumn{3}{c}{D$\rightarrow$A} & \multicolumn{3}{c}{D$\rightarrow$W} & \multicolumn{3}{c}{W$\rightarrow$A} & \multicolumn{3}{c}{W$\rightarrow$D} & \multicolumn{3}{c}{Avg} \\ \cmidrule{2-22}
        & OS & OS$^*$ & Acc & OS & OS$^*$ & Acc & OS & OS$^*$ & Acc & OS & OS$^*$ & Acc & OS & OS$^*$ & Acc & OS & OS$^*$ & Acc & OS & OS$^*$ & Acc \\ \midrule
        pl      & 88.0 & 87.8 & 89.0 & 80.4 & 79.3 & 85.2 & 66.1 & 63.3 & 78.7 & 85.5 & 84.8 & 88.4 & 69.7 & 67.4 & 80.2 & 98.5 & 98.7 & 95.4 & 81.4 & 80.2 & 86.2 \\
        tc      & 97.1 & 98.6 & 89.4 & 94.9 & 95.6 & 91.8 & 92.6 & 94.2 & 85.5 & 98.0 & 99.2 & 93.1 & 92.6 & 94.2 & 85.2 & 97.6 & 100.0 & 85.9 & 95.5 & 97.0 & 88.5 \\
        full    & 96.6 & 97.7 & 90.8 & 95.7 & 96.4 & 92.5 & 93.1 & 93.6 & 90.7 & 97.5 & 98.0 & 95.0 & 92.9 & 93.6 & 89.5 & 99.3 & 100.0 & 95.8 & 95.8 & 96.6 & 92.4 \\ \bottomrule
    \end{tabular}
\end{table}

\begin{table}
	\small
	\centering
	\caption{Ablation study and alternative designs on VisDA-2017 (VGGNet).}
	\label{table:alter}
	\renewcommand{\arraystretch}{0.9}
	\begin{tabular}{lcccccccccc} \toprule
		\multirow{2}{*}{\textbf{Method\textbackslash Class}}  && \multicolumn{9}{c}{Synthetic$\rightarrow$Real} \\ \cmidrule{3-11}
        && bicycle & bus & car & m-cycle & train & truck & OS & OS$^*$ & Acc \\ \midrule
        pl              && 8.8 & 55.3 & 55.3 & 89.7 & 65.6 & 1.3 & 47.8 & 46.0 & 52.7 \\
		tc              && 76.5 & 85.5 & 53.3 & 76.7 & 74.2 & 0.8 & 65.6 & 61.1 & 71.2 \\ 
        full            && 76.4 & 69.4 & 52.1 & 81.5 & 76.7 & 43.6 & 70.8 & 66.6 & 75.6 \\ \midrule
        full (cosine)   && 0.0 & 0.0 & 2.5 & 83.3 & 0.0 & 0.0 & 24.1 & 14.3 & 40.9 \\
        full (max prob) && 76.5 & 75.3 & 49.7 & 74.9 & 76.1 & 41.1 & 69.6 & 65.6 & 74.3 \\ \bottomrule
	\end{tabular}
\end{table}

We do ablation study on both Office-31 and VisDA dataset to show the effect of each component in our model. The proposed model consists of pseudo-label generation (pl entry in the table) and transformation consistency (tc entry in the table) and we compare the full model (full entry in the table) with them. Particularly, they are compared on the metric of OS, OS$^*$ and total accuracy (Acc) for each adaptation scenario to reflect different aspects.

Results on Office-31 are shown in Table~\ref{table:ablation} with standard deviations reported in appendix. From the table we can see overall, the full model performs best. For adaptation scenarios that have small domain gaps (\eg W$\rightarrow$D), the pseudo-label method alone can achieve reasonably good results. In terms of OS and OS$^*$, the transformation consistency objective alone can achieve high scores. However, its total accuracy is significantly lower than the full model. pl improves the total accuracies (Acc) in every adaptation scenario. On VisDA we can see from Table~\ref{table:alter} that pl or tc alone performs significantly worse. This suggests that the full model is better in classifying known instances and detecting unknown instances, and the two objectives are complementary.

\subsubsection{Reliability of pseudo-label}

\label{sec:reliability}

\begin{figure}[t]
	\begin{minipage}{.31\textwidth}
        \centering
        \includegraphics[width=1\textwidth]{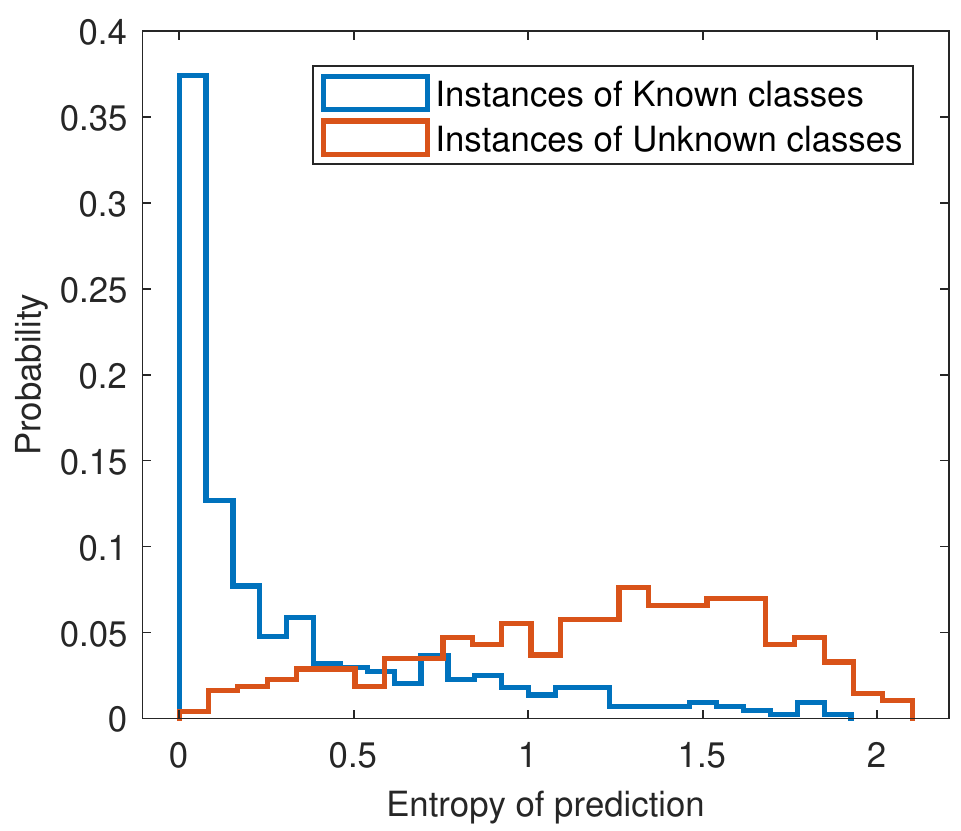}
        \captionof{figure}{Histogram of target instance entropy on Office-31 A$\rightarrow$D.}
        \label{fig:entropy}
	\end{minipage}
    \hfill
    \begin{minipage}{.31\textwidth}
        \centering
        \includegraphics[width=1\textwidth]{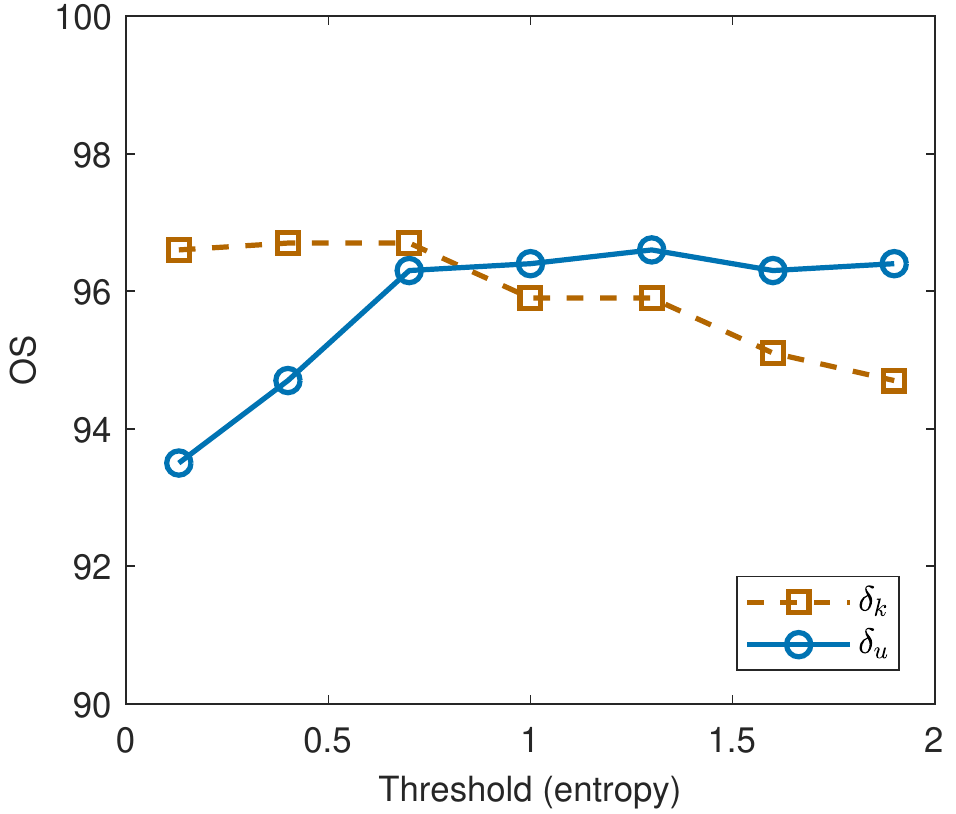}
        \captionof{figure}{\small Sensitivity of $\delta_k$ and $\delta_u$.}
        \label{fig:sens}
    \end{minipage}
    \hfill
	\begin{minipage}{.31\textwidth}
		\centering
		\includegraphics[width=1\textwidth]{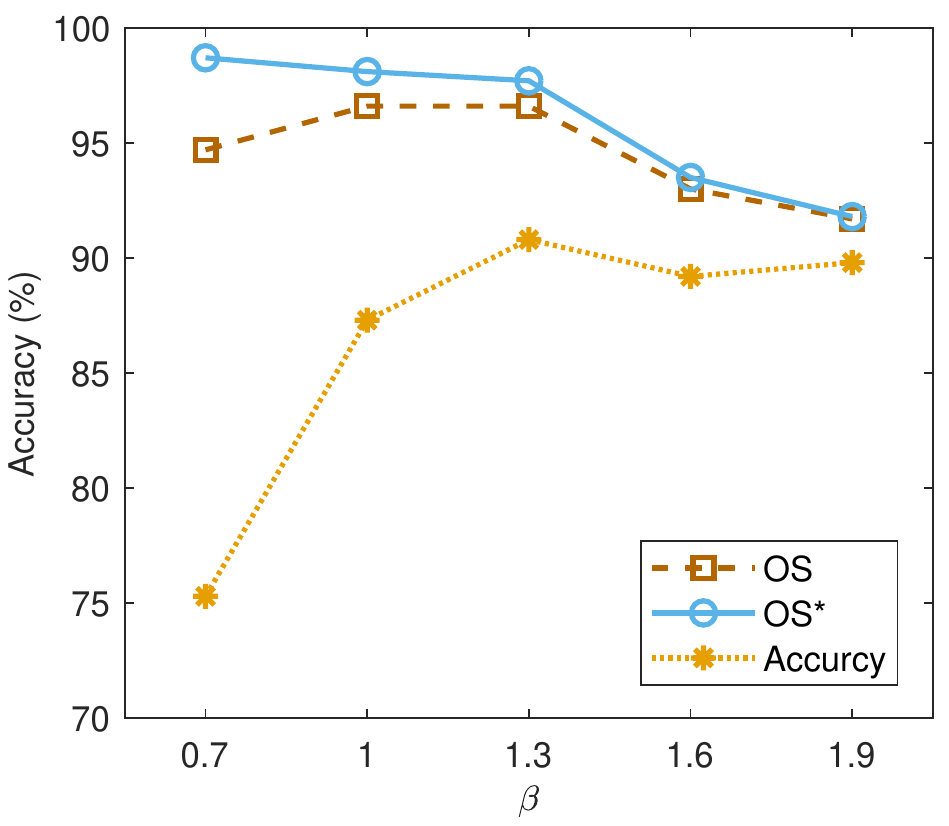}
		\captionof{figure}{Classification accuracies on Office-31 with different $\beta$ value.}
		\label{fig:beta}
	\end{minipage}
\end{figure}

\begin{figure}
	\begin{minipage}{.63\textwidth}
		\begin{subfigure}{.49\textwidth}
			\centering
			\includegraphics[width=1\textwidth]{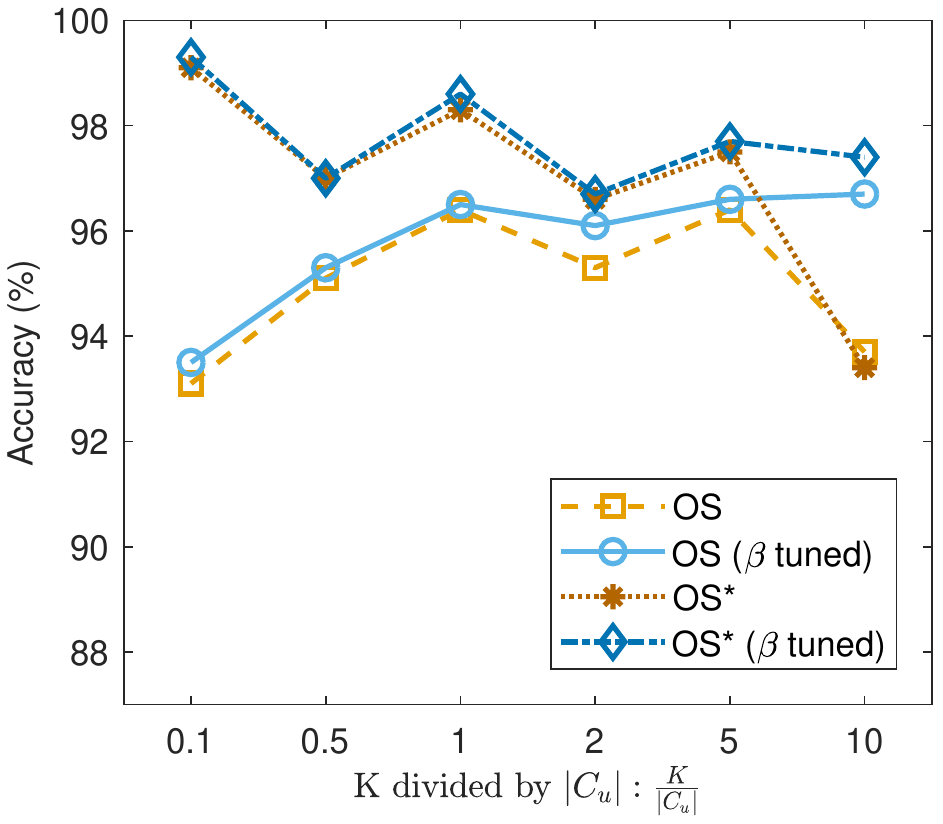}
			\caption{OS and OS$^*$.}
			\label{fig:k1}
		\end{subfigure}
		\hfill
		\begin{subfigure}{.49\textwidth}
			\centering
			\includegraphics[width=1\textwidth]{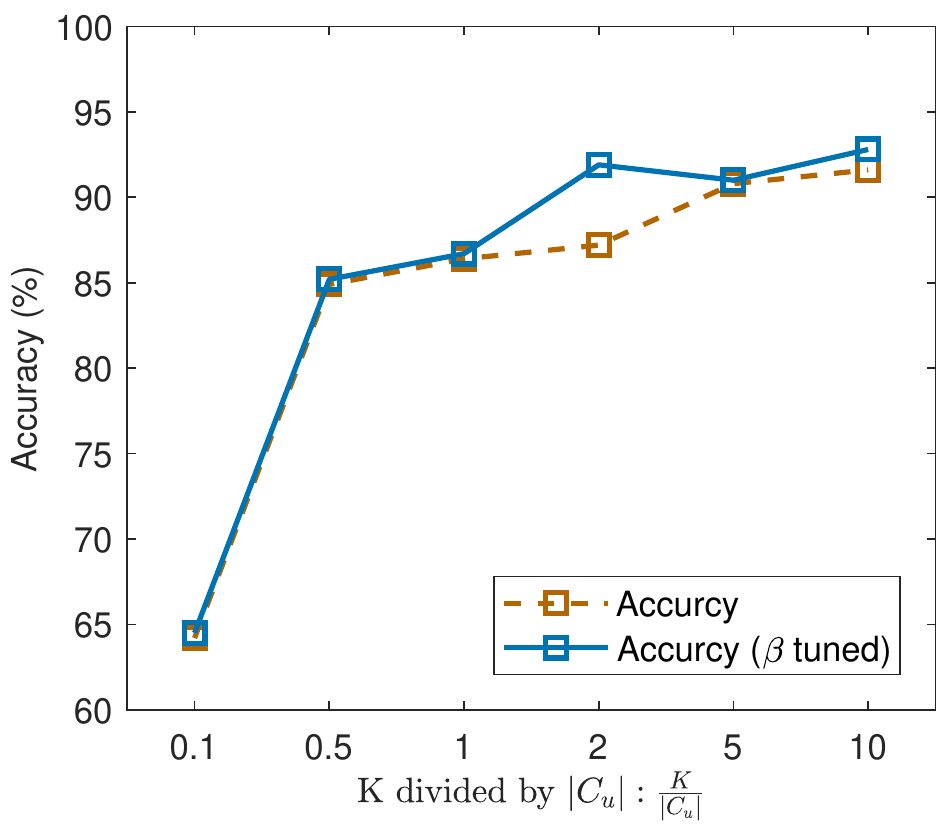}
			\caption{Total accuracy.}
			\label{fig:k2}
		\end{subfigure}
		\caption{Classification accuracies on Office-31 with different $\nicefrac{K}{\lvert C_u\rvert}$.}
		\label{fig:k}
	\end{minipage}
	\hfill
	\begin{minipage}{.31\textwidth}
		\centering
		\includegraphics[width=1\textwidth]{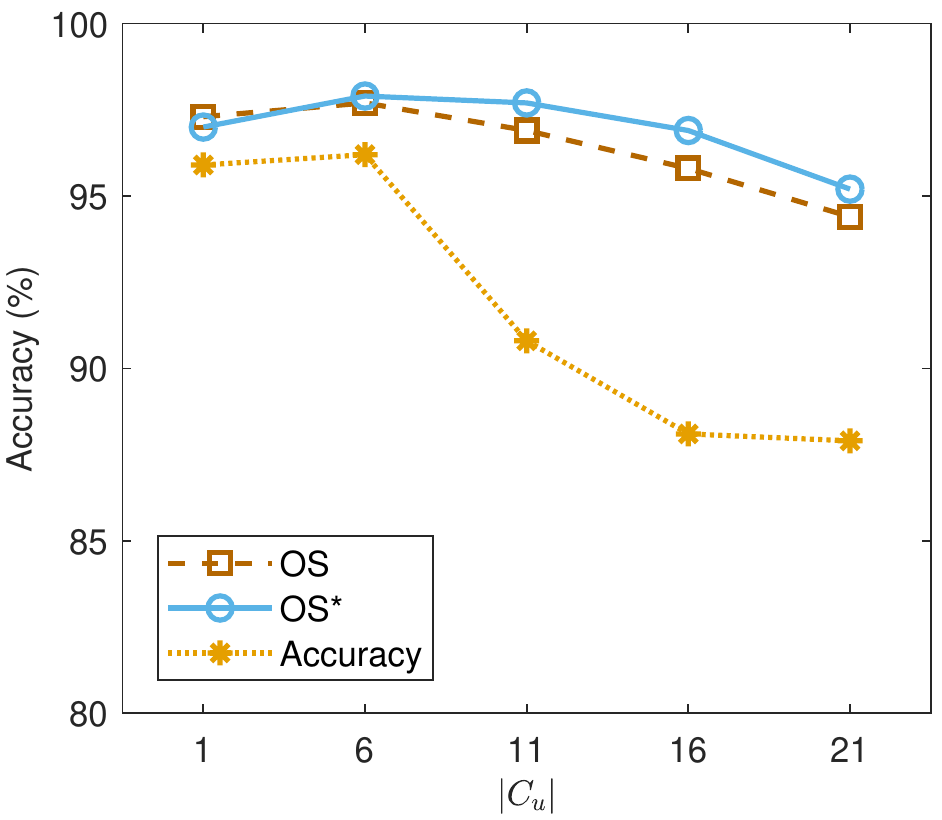}
		\captionof{figure}{Classification accuracies on Office-31 with different openness.}
		\label{fig:open}
	\end{minipage}
\end{figure}

We plot the histogram of $\hat{f}_s$'s prediction entropy on target domain instances with known and unknown classes plotted separately. Figure~\ref{fig:entropy} shows that they are statistically different. Thus, by using a threshold, the selected pseudo-known instances (left) and pseudo-unknown instances (right) are mostly correct. The accuracies of generated pseudo-labels are 96\% for known-class instances and 90\% for the unknown instances, respectively. On VisDA, we can generate pseudo-labels with 77\% and 61\% accuracies. This means even with a large domain gap the selected pseudo-labels are mostly accurate, which verify the reasonableness of using pseudo-labels in UODA.

\subsubsection{Alternative design}

\label{sec:alter}

Another distance measure on transformations is the negative cosine similarity. In contrast to mutual information, it only considers the distance of individual instance pairs rather than statistical dependence. Results in Table~\ref{table:alter} shows that it fails to converge on VisDA in UODA setting. Another measure of confidence is the maximal probability in $\hat{f}_s$'s prediction. When using this criterion, we find that the proposed method can achieve competitive performance as shown in Table~\ref{table:alter}.

\subsubsection{Hyperparameters}

\label{sec:hyper}

\noindent \textbf{Effect of $\beta$}\quad We study the effect of the hyperparameters $\beta$. We found that our model's performance is consistent across domains: if a model performs better on one adaptation scenario then it performs better on other scenarios as well. We report results on Office-31 A$\rightarrow$D scenario for comparison. Figure~\ref{fig:beta} shows the OS, OS$^*$ and total accuracy with respect to $\beta$. It can be seen that the model can achieve a competitive performance when $\beta=1$, which corresponds to the original mutual information formulation. When $\beta$ is small, the model is encouraged to give predictions with low entropy, which will mistakenly predict unknown-class instance into known source classes and decrease the accuracy. This can be seen from $\beta=0.85$ where OS$^*$ is very high while the total accuracy is low. The total accuracy can be further improved when $\beta$ is tuned.

\noindent \textbf{Effect of $K$}\quad In order to study the effect of $K$ we have to relate it to the number of unknown classes $\lvert C_u\rvert=\lvert C_t\rvert-\lvert C_s\rvert$ in the target domain and consider their ratio. Thus, in Figure~\ref{fig:k} we plot the accuracies with respective to $\nicefrac{K}{\lvert C_u\rvert}$, which is chosen to range from 0.1 to 10. $\nicefrac{K}{\lvert C_u\rvert}=0.1$ means the model outputs few scores with respect to $\lvert C_u\rvert$ (which is 1 in the case of Office-31 A$\rightarrow$D scenario) and vice versa. The model has low performance when $\nicefrac{K}{\lvert C_u\rvert}=0.1$, which is expected as in the situation, the model has to classify different unknown-class instances into a single category even they are semantically different, which will confound the training since the training objective is to discriminate every different semantics as much as possible. Experimental results with fixed $\beta$ reveal that performance drops when $K$ is too large. The reason is that the entropy term in our objectives is dependent upon the number of outputs ($\lvert C_s\rvert+K$), thus the optimal $\beta$ may differ. However, when $\beta$ is tuned the proposed method can have a high and stable performance when $K$ is large. Hence, in practice, we can first set a rather large $K$ and then select the value of $\beta$ using model selection techniques like cross-validation.

\noindent \textbf{Sensitivity of threshold $\delta_k$ and $\delta_u$}\quad We perform an analysis on the value of threshold $\delta_k$ and $\delta_u$ on Office-31 A$\rightarrow$D. Results in Figure~\ref{fig:sens} show that increasing $\delta_k$ or decreasing $\delta_u$ too much can hurt the performance.

\subsubsection{Openness}
In this section we examine the model's ability in datasets with varying openness, where the number of unknown classes in target domain changes. Results in Figure~\ref{fig:open} show that in terms of OS and OS$^*$, our method can get stable and high averaged accuracy when $\lvert C_u\rvert$ varies. The total accuracy becomes higher as $\lvert C_u\rvert$ gets smaller, which means the UODA task with less unknown classes are easier.

\subsubsection{Qualitative analysis}

\begin{figure}
    \centering
    \includegraphics[width=0.49\textwidth]{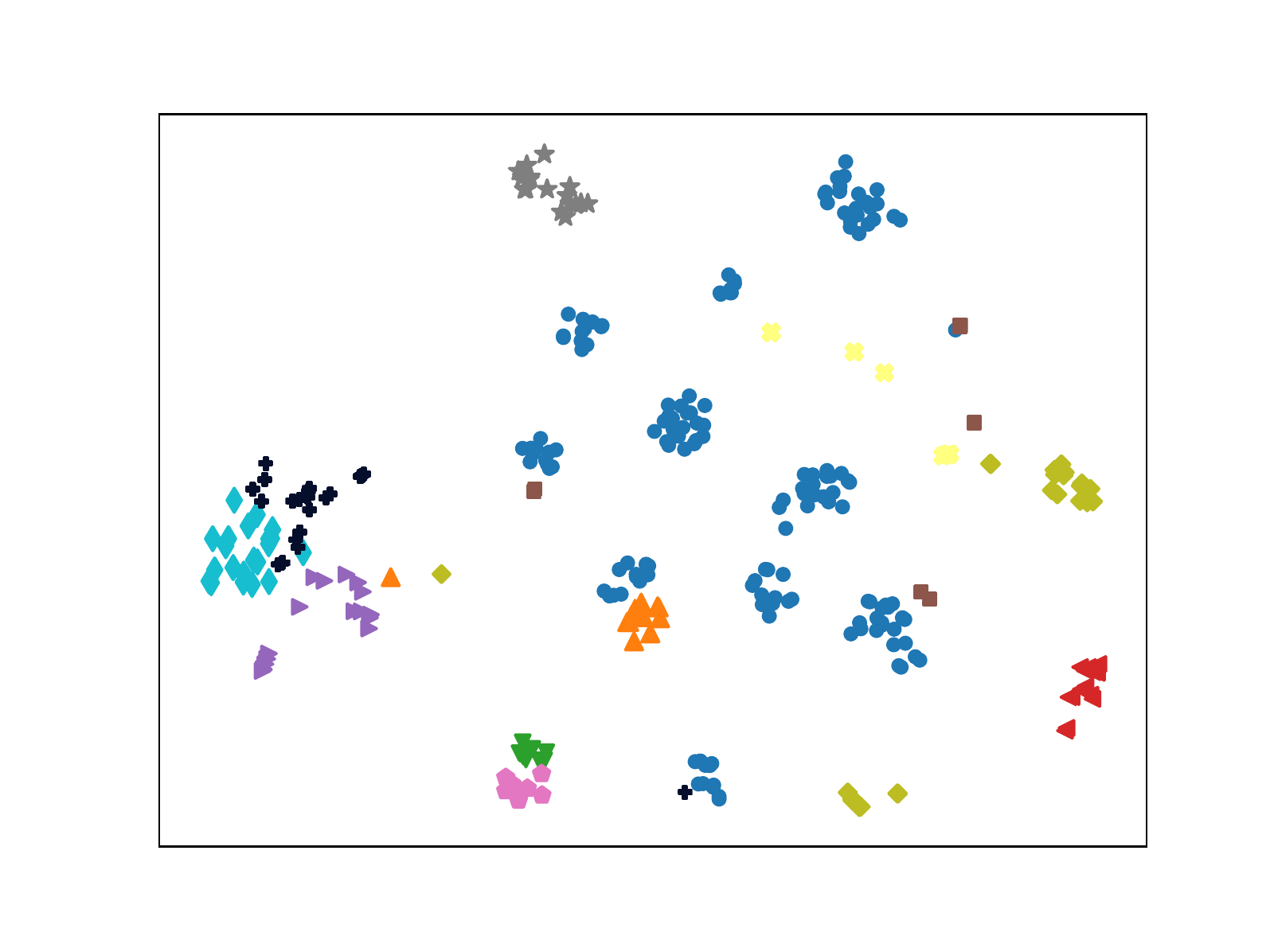}
    \caption{t-SNE embeddings on target domain. Blue dots represent known source classes. Unknown classes are in different colors and shapes.}
    \label{fig:tsne}
\end{figure}

Figure~\ref{fig:tsne} plots the t-SNE~\cite{tsne} embeddings of the pre-logits features on target domain (DSLR) in Office-31 dataset under the A$\rightarrow$D adaptation scenario. Instances belonging to known source classes are shown as blue dots. Different unknown target classes are depicted in different colors and shapes. It can be seen from the figure that our method can classify known source instances well into 10 classes. Furthermore, our method also successfully classify different unknown target classes into different clusters rather than merge them as one. This discrimination on unknown classes is quite accurate even when $K$ is much larger than $C_u$.

\section{Conclusions}
We have presented a novel method for open-set domain adaptation in the hypothesis transfer setting, where only source domain model is provided as apposed to source data. Our method consists of generating pseudo-labels for confident target instances using the pre-trained model and ensuring target prediction consistency of semantically similar pairs by mutual information maximization. This structural semantic information is rarely appreciated by previous domain adaptation methods. The proposed method is easy to optimize and deploy. More knowledge behind the target data can be explored as future work, and together with out-of-distribution training on source domain they are two promising directions for domain adaptation. The proposed solution provides a new formulation of UODA via a semi-supervised clustering perspective. We believe our work would guide various future research in unsupervised domain adaptation as an important baseline.

{\small
\bibliographystyle{plainnat}
\bibliography{main}
}

\clearpage

\appendix

\newtheorem{innercustomthm}{Proposition}
\newenvironment{customthm}[1]
{\renewcommand\theinnercustomthm{#1}\innercustomthm}
{\endinnercustomthm}

\section{Proof of Proposition~\ref{thm:mi}}

In this section we provide the proof of Proposition~\ref{thm:mi}. Here and in the sequel the superscript $t$ is dropped as it is clear from the context.

\begin{innercustomthm}
	Assume that transformation $x_+$ does not lose the information of its true label $y$ and contains no other information in $x$, \ie $x_+=\argmin_{x'}I(x;x'),\,\textrm{s.t.}\, I(x;y)=I(x';y)$, then
	\begin{equation}
	I(\tilde{y};\tilde{y}_+) \leq I(\tilde{y};y).
	\end{equation}
\end{innercustomthm}
\begin{proof}
	Inspired by~\cite{tian2020makes}, we first show that $x_+$ contains same amount information of $y$ as $x$. Since the joint probability distribution of $y$, $x$ and $x_+$ can be written as $p(y,x,x_+)=p(x_+\vert x)p(x\vert y)p(y)$, $(y, x, x_+)$ form a Markov chain $y\rightarrow x\rightarrow x_+$. By data-processing inequality, $I(x;x_+)\geq I(y;x_+)$. Therefore when $x_+$ achieves the minimal value of $I(x;x_+)$
	\begin{equation}
	\label{eq:equal}
	I(x;x_+)=I(x;y)=I(x_+;y).
	\end{equation}
	Then recall that $\tilde{y}$ and $\tilde{y}_+$ are functions of $x$ and $x_+$, respectively. By chain rule for mutual information, we can have
	\begin{equation}
	\label{eq:twoy}
	\begin{split}
	I(x;x_+) &= I(x,\tilde{y};x_+) \\
	&= I(x_+;\tilde{y}) + I(x_+;x\vert \tilde{y}) \\
	&= I(\tilde{y};x_+,\tilde{y}_+) + I(x;x_+\vert\tilde{y}) \\
	&= I(\tilde{y};\tilde{y}_+) + I(\tilde{y};x_+\vert\tilde{y}_+) + I(x;x_+\vert\tilde{y}),
	\end{split}
	\end{equation}
	and
	\begin{equation}
	\label{eq:truey}
	\begin{split}
	I(y;x) &= I(y;x,\tilde{y}) \\
	&= I(y;\tilde{y}) + I(y;x\vert\tilde{y}).
	\end{split}
	\end{equation}
	Next by log sum inequality, we can write
	\begin{equation}
	\label{eq:lsi}
	\begin{split}
	I(y;x\vert\tilde{y}) - I(x;x_+\vert\tilde{y}) &= \sum_{y,x,\tilde{y}}p(y,x,\tilde{y})\log\frac{p(y,x\vert\tilde{y})}{p(y\vert\tilde{y})p(x\vert\tilde{y})} - \sum_{x,x_+,\tilde{y}}p(x,x_+,\tilde{y})\log\frac{p(x,x_+\vert\tilde{y})}{p(x\vert\tilde{y})p(x_+\vert\tilde{y})} \\
	&= \sum_{y,x,x_+,\tilde{y}}p(y,x,x_+,\tilde{y})\log\frac{p(y,x\vert\tilde{y})}{p(y\vert\tilde{y})}\frac{p(x_+\vert\tilde{y})}{p(x,x_+\vert\tilde{y})} \\
	&= \sum_{y,x,x_+,\tilde{y}}p(y,x,x_+,\tilde{y})\log\frac{p(y,x,\tilde{y})}{p(x_+,x,\tilde{y})}\frac{p(x_+,\tilde{y})}{p(y,\tilde{y})} \\
	&= \sum_{y,x,x_+,\tilde{y}}p(y,x,x_+,\tilde{y})\log\frac{p(\tilde{y}\vert x)p(x\vert y)p(y)}{p(\tilde{y}\vert x)p(x\vert x_+)p(x_+)}\frac{p(\tilde{y}\vert x_+)p(x_+)}{p(\tilde{y}\vert y)p(y)} \\
	&= \sum_{y,x,x_+}p(y,x,x_+)\log\frac{p(x\vert y)}{p(x\vert x_+)} + \sum_{y,x,x_+,\tilde{y}}p(y,x,x_+,\tilde{y})\log\frac{p(\tilde{y}\vert x_+)}{p(\tilde{y}\vert y)} \\
	&= \sum_{y,x,x_+}p(y,x,x_+)\log\frac{p(x\vert y)}{p(x\vert x_+)} + \sum_{y,x_+,\tilde{y}}\sum_x p(y,x,x_+,\tilde{y})\log\frac{\sum_x p(\tilde{y},x\vert x_+)}{\sum_x p(\tilde{y},x\vert y)} \\
	&\leq \sum_{y,x,x_+}p(y,x,x_+)\log\frac{p(x\vert y)}{p(x\vert x_+)} + \sum_{y,x,x_+,\tilde{y}}p(y,x,x_+,\tilde{y})\log\frac{p(\tilde{y},x\vert x_+)}{p(\tilde{y},x\vert y)} \\
	&= \sum_{y,x,x_+,\tilde{y}}p(y,x,x_+,\tilde{y})\log\frac{p(x\vert y)}{p(x\vert x_+)}\frac{p(\tilde{y}\vert x)p(x\vert x_+)}{p(x\vert y)p(\tilde{y}\vert x)} \\
	&= 0.
	\end{split}
	\end{equation}
	Therefore, by combining Eq.~(\ref{eq:equal})~(\ref{eq:twoy})~(\ref{eq:truey}) and (\ref{eq:lsi}), we can complete the proof
	\begin{equation}
	\begin{split}
	I(\tilde{y};\tilde{y}_+) &= I(y;\tilde{y}) + I(y;x\vert\tilde{y}) - I(\tilde{y};x_+\vert\tilde{y}_+) - I(x;x_+\vert\tilde{y}) \\
	&\leq I(y;\tilde{y}) - I(\tilde{y};x_+\vert\tilde{y}_+) \\
	&\leq I(y;\tilde{y}).
	\end{split}
	\end{equation}
\end{proof}

\section{Mutual information objective}

\label{sec:app_mi}

In this paper we use the objective $\beta H(\tilde{y}^t)-H(\tilde{y}^t\vert \tilde{y}_+^t)$ that balances between entropy and conditional entropy, which leads to a formulation of mutual information with a $\beta$ coefficient~\cite{Ji_2019_ICCV}
\begin{equation}
    \begin{split}
        I_{\beta}(y_1,y_2) &= \frac{1}{2}\big(\beta H(y_1)-H(y_1\vert y_2)+\beta H(y_2)-H(y_2\vert y_1)\big) \\
                           &= \frac{1}{2}\big(-\beta\sum_{c_1}p(c_1)\log p(c_1)+\sum_{c_1,c_2}p(c_1,c_2)\log(c_1\vert c_2) \\
                           & \quad \quad \quad -\beta\sum_{c_2}p(c_2)\log p(c_2)+\sum_{c_1,c_2}p(c_1,c_2)\log(c_2\vert c_1)\big) \\
                           &= \frac{1}{2}\sum_{c_1,c_2}p(c_1,c_2)\log\frac{p(c_1\vert c_2)}{p(c_1)^{\beta}} + \frac{1}{2}\sum_{c_1,c_2}p(c_1,c_2)\log\frac{p(c_2\vert c_1)}{p(c_2)^{\beta}} \\
                           &= \sum_{c_1,c_2}p(c_1,c_2)\log\frac{p(c_1,c_2)}{\big(p(c_1)p(c_2)\big)^{\frac{\beta+1}{2}}}.
    \end{split}
\end{equation}

\section{Proof of Proposition~\ref{thm:consistency}}

\begin{innercustomthm}
	The estimator $\hat{I}_{\beta}(\tilde{y};\tilde{y}_+)$ of $I_{\beta}(\tilde{y};\tilde{y}_+)$ is consistent.
\end{innercustomthm}
\begin{proof}
	\begin{equation}
	\begin{split}
	\lvert\hat{I}_{\beta}(\tilde{y};\tilde{y}_+) - I_{\beta}(\tilde{y};\tilde{y}_+)\rvert &\leq \sum_{c,c_+} \lvert P_{c,c_+}\log\frac{P_{c,c_+}}{(P_cP_{c_+})^{\frac{\beta+1}{2}}} - p(c,c_+)\log\frac{p(c,c_+)}{(p(c)p(c_+))^{\frac{\beta+1}{2}}}\rvert \\
	&\leq \sum_{c,c_+} \lvert P_{c,c_+}\log P_{c,c_+} - p(c,c_+)\log p(c,c_+)\rvert \\
	& \qquad + \frac{\beta+1}{2}\lvert p(c,c_+)\log p(c)p(c_+) - P_{c,c_+}\log P_cP_{c_+}\rvert.
	\end{split}
	\end{equation}
	According to the law of large numbers we have $P_{c,c_+}\xrightarrow{\text{a.s.}}p(c,c_+)$ and $P_c\xrightarrow{\text{a.s.}}p(c)$. Since $\sigma_c(f(\cdot))>0$ for any $c$, the individual probabilities in the above equation are all positive. Note that both $\log x$ and $x\log x$ are continuous functions for $x>0$. Therefore, by continuous mapping theorem and Slutsky's theorem, we can finally have $\hat{I}_{\beta}(\tilde{y};\tilde{y}_+)\xrightarrow{\text{a.s.}}I_{\beta}(\tilde{y};\tilde{y}_+)$.
\end{proof}

\section{Additional experimental results}

We summarize the additional standard deviations in Table~\ref{table:office31_std},~\ref{table:officehome_std} and~\ref{table:ablation_std}. Results of OS$^*$ are shown in Table~\ref{table:officehome_os}.

\begin{table}
    \small
    \centering
    \caption{Standard deviations of average classification accuracy ($\%$) on Office-31 (ResNet-50).}
    \label{table:office31_std}
    \renewcommand{\tabcolsep}{4.6pt}
    \begin{tabular}{lccccccccccccc} \toprule
        \multirow{2}{*}{\textbf{Method\textbackslash Task}} && \multicolumn{2}{c}{A$\rightarrow$D} & \multicolumn{2}{c}{A$\rightarrow$W} & \multicolumn{2}{c}{D$\rightarrow$A} & \multicolumn{2}{c}{D$\rightarrow$W} & \multicolumn{2}{c}{W$\rightarrow$A} & \multicolumn{2}{c}{W$\rightarrow$D} \\ \cmidrule{3-14}
                        && OS & OS$^*$ & OS & OS$^*$ & OS & OS$^*$ & OS & OS$^*$ & OS & OS$^*$ & OS & OS$^*$ \\ \midrule
        ResNet             						&& $\pm$0.3 & $\pm$0.9 & $\pm$1.2 & $\pm$0.9 & $\pm$1.0 & $\pm$1.1 & $\pm$0.3 & $\pm$0.7 & $\pm$1.0 & $\pm$1.6 & $\pm$0.2 & $\pm$0.4 \\
        RTN~\cite{NIPS2016_6110}                && $\pm$1.4 & $\pm$1.6 & $\pm$1.2 & $\pm$1.0 & $\pm$0.9 & $\pm$1.5 & $\pm$0.3 & $\pm$0.7 & $\pm$0.6 & $\pm$1.4 & $\pm$0.2 & $\pm$0.9 \\
        DANN~\cite{pmlr-v37-ganin15}            && $\pm$0.6 & $\pm$0.6 & $\pm$0.7 & $\pm$1.1 & $\pm$1.6 & $\pm$0.9 & $\pm$0.2 & $\pm$0.5 & $\pm$1.2 & $\pm$0.8 & $\pm$0.1 & $\pm$0.0 \\
        OpenMax~\cite{Bendale_2016_CVPR}       	&& $\pm$0.9 & $\pm$0.9 & $\pm$0.5 & $\pm$0.3 & $\pm$1.0 & $\pm$0.6 & $\pm$0.4 & $\pm$0.3 & $\pm$0.9 & $\pm$0.6 & $\pm$0.3 & $\pm$0.3 \\
        ATI-$\lambda$~\cite{Busto_2017_ICCV}	&& $\pm$1.2 & $\pm$1.1 & $\pm$1.5 & $\pm$1.4 & $\pm$1.8 & $\pm$1.5 & $\pm$1.0 & $\pm$1.0 & $\pm$1.4 & $\pm$1.2 & $\pm$0.9 & $\pm$0.8 \\
        OSBP~\cite{Saito_2018_ECCV}				&& $\pm$1.4 & $\pm$1.3 & $\pm$2.0 & $\pm$2.1 & $\pm$2.5 & $\pm$2.3 & $\pm$1.0 & $\pm$0.4 & $\pm$2.5 & $\pm$1.3 & $\pm$0.9 & $\pm$0.6 \\
        STA~\cite{Liu_2019_CVPR}				&& $\pm$1.5 & $\pm$0.4 & $\pm$0.6 & $\pm$0.5 & $\pm$0.5 & $\pm$0.8 & $\pm$0.2 & $\pm$0.5 & $\pm$0.9 & $\pm$0.6 & $\pm$0.2 & $\pm$0.1 \\ \midrule
        FS~\cite{inheritable_cvpr_2020}			&& $\pm$1.1 & $\pm$0.8 & $\pm$0.7 & $\pm$1.2 & $\pm$0.2 & $\pm$0.2 & $\pm$0.5 & $\pm$0.7 & $\pm$1.3 & $\pm$0.9 & $\pm$0.2 & $\pm$0.3 \\ \midrule
        Ours									&& $\pm$1.1 & $\pm$1.0 & $\pm$0.7 & $\pm$0.9 & $\pm$0.1 & $\pm$0.1 & $\pm$0.9 & $\pm$1.0 & $\pm$0.4 & $\pm$0.3 & $\pm$0.5 & $\pm$0.0 \\ \bottomrule
    \end{tabular}
\end{table}
    
\begin{table}
    \small
    \centering
    \caption{Standard deviations of average classification accuracy OS ($\%$) on Office-Home (ResNet-50).}
    \label{table:officehome_std}
    \renewcommand{\tabcolsep}{3.5pt}
    \begin{tabular}{lcccccccccccccc} \toprule
        \textbf{Method\textbackslash Task} && A$\rightarrow$C & A$\rightarrow$P & A$\rightarrow$R & C$\rightarrow$A & C$\rightarrow$P & C$\rightarrow$R & P$\rightarrow$A & P$\rightarrow$C & P$\rightarrow$R & R$\rightarrow$A & R$\rightarrow$C & R$\rightarrow$P \\ \midrule
        ResNet									&& $\pm$0.4 & $\pm$0.7 & $\pm$0.6 & $\pm$0.6 & $\pm$0.5 & $\pm$0.6 & $\pm$0.3 & $\pm$0.6 & $\pm$0.9 & $\pm$0.3 & $\pm$0.5 & $\pm$0.4 \\
        ATI-$\lambda$~\cite{Busto_2017_ICCV}	&& $\pm$1.2 & $\pm$1.1 & $\pm$0.7 & $\pm$1.2 & $\pm$1.5 & $\pm$0.7 & $\pm$0.9 & $\pm$1.6 & $\pm$1.6 & $\pm$0.5 & $\pm$1.4 & $\pm$1.5 \\
        DANN~\cite{pmlr-v37-ganin15}			&& $\pm$0.7 & $\pm$1.1 & $\pm$0.8 & $\pm$1.2 & $\pm$1.0 & $\pm$0.4 & $\pm$1.0 & $\pm$1.6 & $\pm$0.4 & $\pm$1.0 & $\pm$1.4 & $\pm$0.8 \\
        OSBP~\cite{Saito_2018_ECCV}				&& $\pm$1.9 & $\pm$1.5 & $\pm$0.9 & $\pm$2.0 & $\pm$1.5 & $\pm$2.2 & $\pm$1.1 & $\pm$2.1 & $\pm$1.9 & $\pm$1.1 & $\pm$2.4 & $\pm$1.5 \\
        OpenMax~\cite{Bendale_2016_CVPR}		&& $\pm$0.4 & $\pm$0.3 & $\pm$0.8 & $\pm$0.9 & $\pm$0.4 & $\pm$0.5 & $\pm$0.8 & $\pm$0.7 & $\pm$0.3 & $\pm$0.8 & $\pm$0.4 & $\pm$0.6 \\
        STA~\cite{Liu_2019_CVPR}				&& $\pm$0.6 & $\pm$1.2 & $\pm$0.2 & $\pm$0.5 & $\pm$1.0 & $\pm$0.4 & $\pm$0.8 & $\pm$0.9 & $\pm$0.3 & $\pm$1.0 & $\pm$1.0 & $\pm$0.5 \\ \midrule
        FS~\cite{inheritable_cvpr_2020}			&& $\pm$0.7 & $\pm$1.4 & $\pm$0.9 & $\pm$0.6 & $\pm$1.7 & $\pm$1.3 & $\pm$1.3 & $\pm$1.0 & $\pm$1.4 & $\pm$0.9 & $\pm$1.7 & $\pm$0.6 \\ \midrule
        Ours            						&& $\pm$0.6 & $\pm$0.4 & $\pm$0.7 & $\pm$0.4 & $\pm$0.6 & $\pm$0.5 & $\pm$0.3 & $\pm$0.5 & $\pm$0.3 & $\pm$0.2 & $\pm$0.5 & $\pm$0.1 \\ \bottomrule
    \end{tabular}
\end{table}

\begin{table}
    \small
    \centering
	\caption{Average classification accuracy OS$^*$ ($\%$) of our method on Office-Home (ResNet-50).}
	\label{table:officehome_os}
    \renewcommand{\tabcolsep}{2.5pt}
    \begin{tabular}{lccccccccccccccc} \toprule
        \textbf{Method\textbackslash Task} && A$\rightarrow$C & A$\rightarrow$P & A$\rightarrow$R & C$\rightarrow$A & C$\rightarrow$P & C$\rightarrow$R & P$\rightarrow$A & P$\rightarrow$C & P$\rightarrow$R & R$\rightarrow$A & R$\rightarrow$C & R$\rightarrow$P & Avg \\ \midrule
        Ours (OS$^*$)		&& 61.8 & 82.5 & 88.9 & 73.9 & 76.5 & 84.3 & 73.9 & 60.6 & 86.0 & 79.3 & 63.2 & 83.8 & 76.2 \\
        Ours (OS$^*$, std)	&& $\pm$0.6 & $\pm$0.2 & $\pm$0.5 & $\pm$0.5 & $\pm$0.5 & $\pm$0.6 & $\pm$1.1 & $\pm$0.7 & $\pm$0.3 & $\pm$0.4 & $\pm$0.6 & $\pm$0.2 & - \\ \bottomrule
    \end{tabular}
\end{table}

\begin{table}
    \small
    \centering
	\caption{Standard deviations of ablation study on Office-31 (ResNet-50).}
	\label{table:ablation_std}
    \renewcommand{\tabcolsep}{1.5pt}
    \begin{tabular}{lccccccccccccccccccc} \toprule
        \multirow{2}{*}{\textbf{M}} & \multicolumn{3}{c}{A$\rightarrow$D} & \multicolumn{3}{c}{A$\rightarrow$W} & \multicolumn{3}{c}{D$\rightarrow$A} & \multicolumn{3}{c}{D$\rightarrow$W} & \multicolumn{3}{c}{W$\rightarrow$A} & \multicolumn{3}{c}{W$\rightarrow$D} \\ \cmidrule{2-19}
        & OS & OS$^*$ & Acc & OS & OS$^*$ & Acc & OS & OS$^*$ & Acc & OS & OS$^*$ & Acc & OS & OS$^*$ & Acc & OS & OS$^*$ & Acc \\ \midrule
        pl      & $\pm$1.4 & $\pm$1.4 & $\pm$0.6 & $\pm$0.7 & $\pm$0.7 & $\pm$0.8 & $\pm$1.8 & $\pm$1.9 & $\pm$1.3 & $\pm$0.5 & $\pm$0.6 & $\pm$0.6 & $\pm$1.1 & $\pm$1.1 & $\pm$1.6 & $\pm$0.2 & $\pm$0.3 & $\pm$1.8 \\
        tc      & $\pm$0.6 & $\pm$0.7 & $\pm$0.7 & $\pm$1.1 & $\pm$1.0 & $\pm$1.4 & $\pm$0.4 & $\pm$0.0 & $\pm$2.4 & $\pm$0.0 & $\pm$0.0 & $\pm$0.2 & $\pm$0.1 & $\pm$0.1 & $\pm$1.0 & $\pm$0.3 & $\pm$0.0 & $\pm$2.0 \\
        full    & $\pm$1.1 & $\pm$1.0 & $\pm$1.8 & $\pm$0.7 & $\pm$0.9 & $\pm$0.8 & $\pm$0.1 & $\pm$0.1 & $\pm$0.2 & $\pm$0.9 & $\pm$1.0 & $\pm$1.4 & $\pm$0.4 & $\pm$0.3 & $\pm$1.1 & $\pm$0.5 & $\pm$0.0 & $\pm$2.9 \\ \bottomrule
    \end{tabular}
\end{table}

\end{document}